\documentclass[letterpaper]{article} % DO NOT CHANGE THIS
\usepackage{aaai23}  % DO NOT CHANGE THIS
\usepackage{times}  % DO NOT CHANGE THIS
\usepackage{helvet}  % DO NOT CHANGE THIS
\usepackage{courier}  % DO NOT CHANGE THIS
\usepackage[hyphens]{url}  % DO NOT CHANGE THIS
\usepackage{graphicx} % DO NOT CHANGE THIS
\usepackage{natbib}  % DO NOT CHANGE THIS AND DO NOT ADD ANY OPTIONS TO IT
\usepackage{caption} % DO NOT CHANGE THIS AND DO NOT ADD ANY OPTIONS TO IT
\usepackage{color,xcolor}
\usepackage{colortbl}
\usepackage{booktabs}
\usepackage[lined,boxed,commentsnumbered,ruled]{algorithm2e}
\usepackage{subcaption}
\usepackage{siunitx}
\usepackage{amsmath}
\usepackage{amsthm}
\usepackage{amssymb}
\DeclareMathOperator*{\argmax}{arg\,max}

\newtheorem{theorem}{Theorem}
\newtheorem{lemma}{Lemma}

\newtheorem{definition}{Definition}

\newtheorem{fact}{Fact}

\title{Reliable Robustness Evaluation via Automatically Constructed Attack Ensembles}
\author {
    Shengcai Liu,\textsuperscript{\rm 1,\rm 2}
    Fu Peng, \textsuperscript{\rm 2}
    Ke Tang \textsuperscript{\rm 1,\rm2}\thanks{Corresponding Author.}
}
\affiliations {
    \textsuperscript{\rm 1} Research Institute of Trustworthy Autonomous Systems,\\
    Southern University of Science and Technology, Shenzhen 518055, China\\
    \textsuperscript{\rm 2} Department of Computer Science and Engineering,\\
    Southern University of Science and Technology, Shenzhen 518055, China\\
    liusc3@sustech.edu.cn, pengf2022@mail.sustech.edu.cn, tangk3@sustech.edu.cn
}

\begin{document}

\maketitle

\begin{abstract}
Attack Ensemble (AE), which combines multiple attacks together, provides a reliable way to evaluate adversarial robustness.
In practice, AEs are often constructed and tuned by human experts, which however tends to be sub-optimal and time-consuming.
In this work, we present AutoAE, a conceptually simple approach for automatically constructing AEs.
In brief, AutoAE repeatedly adds the attack and its iteration steps to the ensemble that maximizes ensemble improvement per additional iteration consumed.
We show theoretically that AutoAE yields AEs provably within a constant factor of the optimal for a given defense.
We then use AutoAE to construct two AEs for $l_{\infty}$ and $l_2$ attacks, and apply them without any tuning or adaptation to 45 top adversarial defenses on the RobustBench leaderboard.
In all except one cases we achieve equal or better (often the latter) robustness evaluation than existing AEs, and notably, in 29 cases we achieve better robustness evaluation than the best known one.
Such performance of AutoAE shows itself as a reliable evaluation protocol for adversarial robustness, which further indicates the huge potential of automatic AE construction.
Code is available at \url{https://github.com/LeegerPENG/AutoAE}.
\end{abstract}

\section{Introduction}

Deep neural networks (DNNs) have exhibited vulnerability to adversarial examples, which are crafted by maliciously perturbing the original input \citep{SzegedyZSBEGF14}.
Such perturbations are nearly imperceptible to humans but can fool DNNs into producing unexpected behavior, thus raising major concerns about their utility in security-critical applications \citep{LiuLCT22,dai2022saliency}.
Over the past few years, defense strategies against adversarial attacks have attracted rapidly increasing research interest \citep{Machado2021}.
Among them the notable ones include adversarial training \citep{MadryMSTV18} that uses different losses \citep{Carlini017,ZhangW19} and generates additional training data \citep{CarmonRSDL19,AlayracUHFSK19}, as well as provable robustness \citep{RuanWSHKK19,CroceA019,GowalDS2019,ZhangYJXGJ19}.

\begin{figure}[tbp]
	\centering
	\scalebox{0.8}{
		\includegraphics[width=\columnwidth]{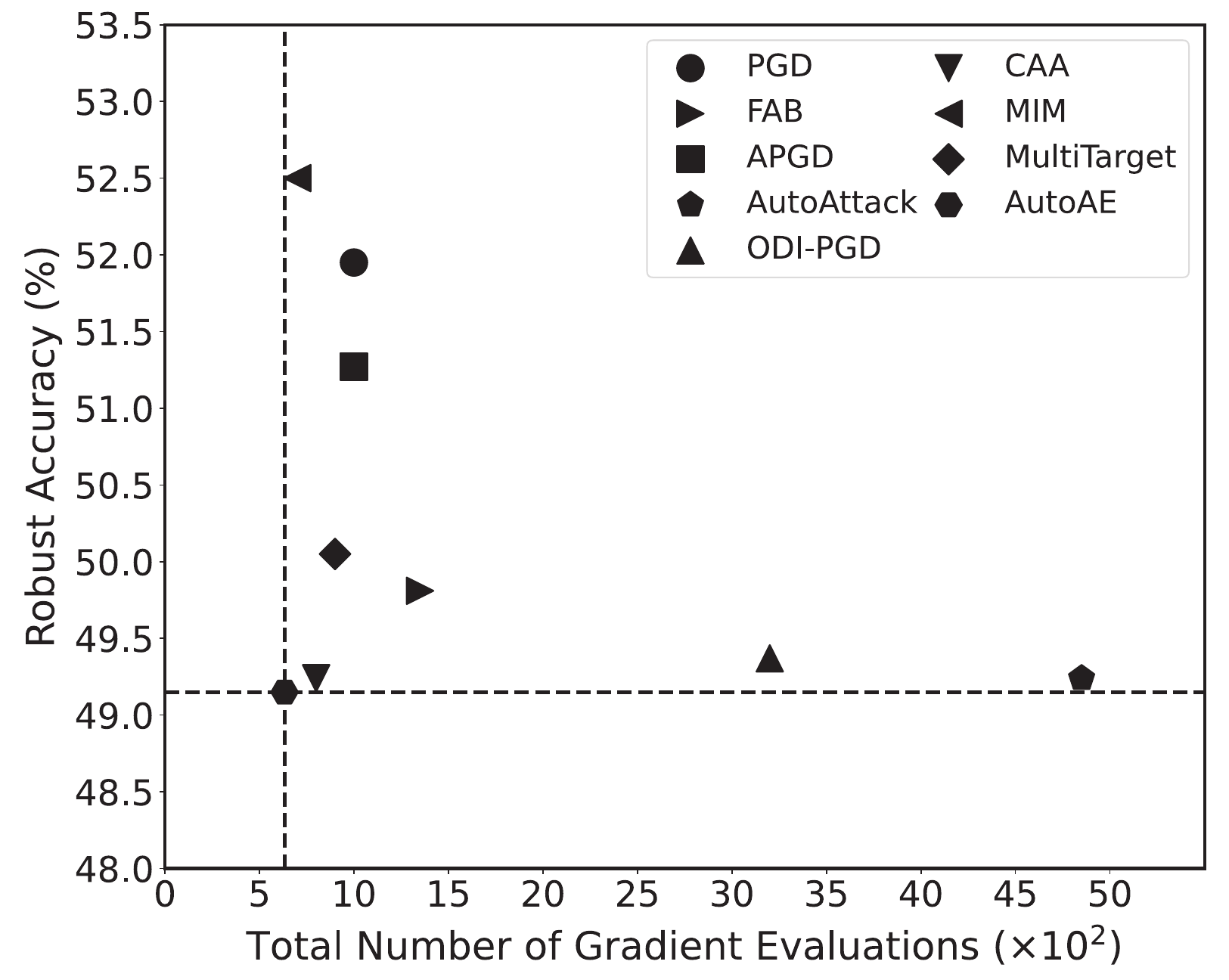}}
	\caption{Comparison of AutoAE and the recently proposed attacks on the CIFAR-10 adversarial training model. AutoAE achieves the best robustness evaluation with only a small number of gradient evaluations.}
	\label{fig:comparison}
\end{figure}

On the other hand, there also exists much evidence showing that many proposed defenses seem to be effective initially but broken later, i.e., exhibiting significant drops of robust accuracy (often by more than 10\%) when evaluated by more powerful or adapted attacks \citep{Carlini017,AthalyeC018}.
With the fast development of this field, it is therefore of utmost importance to reliably evaluate new defenses, such that the really effective ideas can be identified.
Recently, \citet{Croce020a} proposed to assess adversarial robustness using an attack ensemble (AE), dubbed AutoAttack, and has achieved better robustness evaluation for many defenses than the original evaluation.
More specifically, when applied to a defense, an AE would run its component attacks individually, and choose the best of their outputs.
It is thus conceivable that the performance of an AE heavily relies on the complementarity and diversity among its component attacks.
In general, one would choose those attacks with different nature as the component attacks \citep{Croce020a}.
However, given the rich literature of attacks, manual construction of high-quality AEs such as AutoAttack is often a challenging task, requiring domain experts (with deep understanding of both attacks and defenses) to explore the vast design space of AEs, which can be time-consuming and sub-optimal.

The main goal of this work is to develop an easy-to-use approach that not only can reliably evaluate adversarial robustness, but also can significantly reduce the human efforts required in building AEs.
Specifically, we present AutoAE, a simple yet efficient approach for the automatic construction of AEs.
The approach is implemented with a design space containing several candidate attacks; for each attack, there is one hyper-parameter, i.e., iteration steps.
To construct an AE, AutoAE starts from an empty ensemble, and repeatedly adds the candidate attack and its iteration steps to the ensemble that maximizes ensemble improvement per additional iteration spent.

In addition to its conceptual simplicity, AutoAE is appealing also due to its theoretical performance guarantees.
Given a defense model, the AE crafted by AutoAE simultaneously achieves $(1-1/e)$-approximation for the success rate achievable at any given time $t\leq T$ (where $T$ is the total iteration steps consumed by the AE), as well as 4-approximation for the optimal iteration steps.

It is worth mentioning that there also exists another attempt on automatically ensembling attacks.
\citet{MaoCWSHX21} proposed an approach named Composite Adversarial Attack (CAA) that uses multi-objective NSGA-II algorithm to search for attack policies with high attack success rates and low iteration steps.
Here an attack policy also consists of multiple attacks, but differs from AEs in that it is a serial attack connection, where the output of previous attack is used as the initialization input for the successor.
Since the perturbations computed by attacks are accumulated sequentially, it is necessary for CAA to clip or re-scale the perturbations when they exceed the $l_p$-norm constraints.
%which however may impair the attack performance.
Note such repair procedure is unnecessary for AutoAE because the component attacks in an AE are run individually.
The main limitation of CAA is that, due to its multi-objective optimization scheme, the algorithm will output a set of non-dominated attack policies;
therefore one needs to select a final policy from these policies by weighting multiple objectives, while the best weighting parameter is difficult to determine.
Besides, for CAA the policy length is also a parameter that needs to be set by users, while in AutoAE, the ensemble size is automatically determined.
%Finally, unlike AutoAE, the performance of CAA is not theoretically guaranteed. 

We conduct a large-scale evaluation to assess whether AutoAE can reliably evaluate adversarial robustness (see Section~\ref{sec:exp}).
Specifically, we use AutoAE to construct two AEs with a CIFAR-10 $l_{\infty}$- ($\epsilon$=$8/255$) and a CIFAR-10 $l_2$- ($\epsilon$=$0.5$) adversarial training model, and then apply them to 45 top defense models on the RobustBench leaderboard \citep{croce2021robustbench}.
Although using only one restart for each component attack, these two AEs consistently achieve high success rates across various defense models trained on different datasets (CIFAR-10, CIFAR-100 and ImageNet) with different constraints (e.g., $\epsilon$=$8/255$ and $\epsilon=4/255$ for $l_{\infty}$ attack).
In particular, in all except one cases our AEs achieve equal or better (often the latter) robustness evaluation than existing AEs (AutoAttack and CAA); notably, in 29 cases our AEs achieve \textit{new} best known robustness evaluation.
Such performance is achieved by these AEs without any tuning or adaptation to any particular defense at hand, suggesting that they are well suited as minimal tests for any new defense.

We do not argue that AutoAE is the ultimate adversarial attack and this work is not intended for boosting the state-of-the-art attack performance.
As aforementioned, our main purpose is develop a easy-to-use technique for automatic construction of high-quality AEs  that can reliably evaluate adversarial robustness.
As a result, AutoAE can significantly reduce human efforts in AE construction, and also enables researchers to easily obtain their own AEs to meet specific demands.

\section{Preliminaries and Related Works}
\subsection{Adversarial Attack}
\label{sec:preliminary}
Let $\mathcal{F}: [0,1]^D \rightarrow \mathbb{R}^K$ be a $K$-class image classifier taking decisions according to $\argmax_{k=1,...,K}\mathcal{F}_k(\cdot)$ and $x\in[0,1]^D$ be an input image which is correctly classified by $\mathcal{F}$ as $y$.
Given a metric $d(\cdot,\cdot)$ and $\epsilon>0$, the goal of adversarial attack is to find an adversarial example $z$ such that the target model makes wrong predictions on $z$:
\begin{equation}
\label{eq:adv_exam}
\argmax_{k=1,...,K}\mathcal{F}_k(z) \neq y\ \ \text{s.t.}\ \ d(x,z)\leq \epsilon \land z\in[0,1]^D. 
\end{equation}
The problem in Eq.~(\ref{eq:adv_exam}) can be rephrased as maximizing some loss function $L$ to enforce $z$ not to be assisigned to class $y$:
\begin{equation}
\label{eq:loss_function}
\max_{z\in[0,1]^D}L(\mathcal{F}_k(z), y)\ \ \text{s.t.}\ \ d(x,z)\leq \epsilon \land z\in[0,1]^D.
\end{equation}
In image classification, the commonly considered metrics are based on $l_{p}$-distances, i.e., $d(x,z):=||x-z||_p$.
Since the projection on the $l_p$-ball for $p\in\{2,\infty\}$ is available in close form, the Projected Gradient Descent (PGD) attack \citep{MadryMSTV18}, currently the most popular white-box attack, iteratively performs updates on the adversarial example $z$ along the direction of the gradient of loss function.

Hence, the computational complexity of an attack could be characterized by the number of times it computes the gradient of the target model, which typically equals to the number of iteration steps.
In this work, we use the number of gradient evaluations (number of iteration steps) as the complexity metric for attacks and AEs.

\subsection{Existing Attack Ensembles and RobustBench}
\label{sec:related}
This work is mostly inspired by AutoAttack \citep{Croce020a}, which improves the evaluation of adversarial defenses by using an ensemble of four fixed attacks, $\text{APGD}_\text{CE}$, $\text{APGD}_\text{DLR}$, $\text{FAB}$ \citep{Croce020b} and Square Attack \citep{AndriushchenkoC20}.
The key property of AutoAE lies in the diversity among its component attacks with different nature.
%to the strong performance of AutoAttack is the two variants of APGD, which enhances PGD by halving the step size adaptively based on the loss at each step.
In this work, we improve the manual construction of AutoAttack to an automatic construction scheme AutoAE, where the attacks in the former are used to implement the design space (see Section~\ref{exp_constructAE}).

As aforementioned, CAA \citep{MaoCWSHX21} is a multi-objective optimization based approach for automatically ensembling attacks in a serial connection, i.e., attack policy.
In spirit, AutoAE and CAA are similar because they both seek to achieve a high success rate with low iteration steps.
In particular, CAA optimizes these two values as separate objectives, while AutoAE optimizes them simultaneously by maximizing the success rate gain per additional iteration spent.
Compared to CAA, AutoAE is simpler and more straightforward, with theoretical performance guarantees on both success rate and iteration steps.
In practice, AutoAE is easier to use than CAA because the latter requires users to set objective-weighting parameter and policy length.
Besides, the experiments show that AutoAE is more efficient than CAA, i.e., it consumes fewer GPU hours to construct AEs with stronger performance (see Figure~\ref{fig:per_progress}).

In this work, we use RobustBench \citep{croce2021robustbench} as the testbed, which is a standardized adversarial robustness benchmark that maintains a leaderboard based on more than 120 defense models.
It provides a convenient way to track the progress in this field.
%as well as user-friendly interface to use the defense models.
For each defense, the leaderboard presents the robust test accuracy assessed by AutoAttack and its best known accuracy.
%To fully assess AutoAE, we evaluate AutoAE with the top defense models listed on the leaderboard.

Finally, this work is also inspired by the recent advances in AutoML, such as in the domain of neural architecture search \citep{ElskenMH19}.
These automation techniques can effectively help people get rid of the tedious process of architecture design.
This is also true in our work, where searching for high-quality AEs can help better evaluate adversarial robustness.

\section{Methods}
We first formulate the AE construction problem and then detail the AutoAE approach, as well as the theoretical analysis.
%We first formulate the AE construction problem.
\subsection{Problem Formulation}
\label{sec:prob_form}
Given a classifier $\mathcal{F}$ and an annotated training set $D$, where $\mathcal{F}$ can correctly classify the instances in $D$.
Suppose we have a candidate attack pool $\mathbb{S}=\{\mathcal{A}_1,...,\mathcal{A}_n\}$, which can be constructed by collecting existing attacks.
%where each attack has a hyper-parameter, i.e., iteration steps.
The goal is to automatically select attacks from $\mathbb{S}$ to form an AE with strong performance.
On the other hand, as a common hyper-parameter in attacks, the iteration steps can have a significant impact on the attack performance \citep{MadryMSTV18}.
Hence, we set an upper bound $\mathcal{T}$ on the total iteration steps of all attacks in the AE, and also integrate the attack's iteration steps into problem formulation.
Concretely, the AE, denoted as $\mathbb{A}$, is a list of $\langle\text{attack, iteration steps}\rangle$ pairs, i.e., $\mathbb{A}=[\langle \mathcal{A}_1, t_1 \rangle,...,\langle \mathcal{A}_m,t_m \rangle ]$, where $\mathcal{A}_i \in \mathbb{S}$, $t_i \in\mathbb{Z}^+$ is the iteration steps of $\mathcal{A}_i$, $m$ is the ensemble size and $\sum_{i=1}^m t_i \leq \mathcal{T}$.

We denote the set of all possible pairs by $W$, i.e., $W=\{\langle \mathcal{A}, t \rangle|\mathcal{A}\in\mathbb{S} \land t\in \mathbb{Z}^+ \land t\leq \mathcal{T}\}$, and denote 
%Let $[\mathcal{T}]$ denote $\{1,...,\mathcal{T}\}$ and $V$ be the set of all possible AEs that can be constructed based on $\mathbb{S} \times [\mathcal{T}]$.
the AE that results from appending (adding) a pair $\langle \mathcal{A}, t \rangle$ and another AE $\mathbb{A}'$ to $\mathbb{A}$ by $\mathbb{A}\oplus \langle \mathcal{A}, t \rangle$ and $\mathbb{A}\oplus \mathbb{A}'$, respectively.
Generally, an attack $\mathcal{A}$ can be considered as an operation that transforms an input image $x$ to an adversarial example $z$ which is within the $\epsilon$-radius $l_p$ ball around $x$ (see Section~\ref{sec:preliminary}), i.e., $\mathcal{A}: [0,1]^D \rightarrow \{ z|z \in [0,1]^D \land ||z-x||_p \leq \epsilon \}$ .
Supposing $\mathcal{A}$ runs with iteration steps $t$ to attack $\mathcal{F}$ on the instance $\pi=(x,y) \in D$ and generates an adversarial example $z$, we define $Q\left( \langle \mathcal{A},t \rangle,\pi\right)$ as an indicator function of whether $z$ can successfully fool the target model:
%We define $m\left( \langle \mathcal{A},t \rangle,\pi; \mathcal{F} \right)$ as an indicator of the result when running attack $\mathcal{A}$ with iteration steps $t$ for attacking $\mathcal{F}$ on the instance $\pi=(x,y) \in D$:
\begin{equation*}
 Q\left( \langle \mathcal{A},t \rangle, \pi \right) := \mathbb{I}_{\argmax_{k=1,...,K}\mathcal{F}_k(z) \neq y}(z).
\end{equation*}
When using an ensemble $\mathbb{A}=[\langle \mathcal{A}_1, t_1 \rangle,...,\langle \mathcal{A}_m,t_m \rangle ]$ to attack $\mathcal{F}$ on $\pi$, all the attacks in $\mathbb{A}$ are run individually and the best of their outputs is returned.
In other words, the performance of $\mathbb{A}$ on $\pi$, denoted as $Q(\mathbb{A},\pi)$, is:
\begin{equation*}
	Q(\mathbb{A},\pi)=\max_{i=1,...,m} Q\left( \langle \mathcal{A}_i,t_i \rangle, \pi \right).
\end{equation*}
Then the success rate of $\mathbb{A}$ for attacking $\mathcal{F}$ on dataset $D$, denoted as $Q(\mathbb{A},D)$, is:
\begin{equation}
	\label{eq:Q}
	Q(\mathbb{A},D)=\frac{1}{|D|}\sum_{\pi \in D}Q(\mathbb{A},\pi).
\end{equation}
Note $Q(\mathbb{A},D)$ is always 0 when $\mathbb{A}$ is empty.
To make $Q(\cdot,\cdot)$ deterministic, in this work we fix the random seeds of all attacks in $\mathbb{S}$ to keep their outputs stable.

%In addition to success rate, 
Another common concern of an attack is the number of iteration steps, which indicates the computational complexity of it.
Ideally, an AE should achieve the best possible success rate (i.e., 100\%) using as few iteration steps as possible.
Let $\mathbb{A}(t)$ denote the AE resulting from truncating $\mathbb{A}$ at iteration steps $t$, e.g., for $\mathbb{A}=[\langle \mathcal{A}_1,4\rangle, \langle\mathcal{A}_2,4\rangle]$, $\mathbb{A}(5)=[\langle \mathcal{A}_1,4\rangle, \langle\mathcal{A}_2,1\rangle]$.
We use the following expected iteration steps (denoted as $C(\mathbb{A},D)$) needed by $\mathbb{A}$ to achieve 100\% success rate on dataset $D$:
%where the weights are the differences between the current success rate and 100\%:
\begin{equation}
	\label{eq:C}
	C(\mathbb{A},D)= \sum_{t=0}  1-Q(\mathbb{A}(t),D).
\end{equation}
Finally, the AE construction problem is presented in Definition~\ref{def:prob}.
Note the ensemble size is not predefined and should be determined by the construction algorithm.
Since $D$ is fixed for $Q(\cdot,D)$ and $C(\cdot,D)$, for notational simplicity, henceforth we omit the $D$ in them and directly use $Q(\cdot)$ and $C(\cdot)$.
The problem in Definition~\ref{def:prob} is NP-hard, because maximizing only $Q$ is equivalent to solving the NP-hard subset selection problem with general cost constraints \citep{NemhauserW78,QianSYT17,QIAN2019279}. 
%where the ground set is $W$ (see Appendix~A).
%can be reformulated as the NP-hard subsect selection problem with general cost constraints \citep{QianSYT17}, where the ground set is the Cartesian product of $\mathbb{S}$ and $\{1,2,...,\mathcal{T}\}$ (see Appendix~A for details).
\begin{definition}
	\label{def:prob}
	Given $\mathcal{F}$, $D$, $\mathbb{S}$ and $\mathcal{T}$, the AE construction problem is to find $\mathbb{A}^*=(\langle \mathcal{A}^*_1, t^*_1 \rangle,...,\langle \mathcal{A}^*_{m^*},t^*_{m^*} \rangle )$ that maximizes $Q(\mathbb{A}^*,D)$ and minimizes $C(\mathbb{A}^*,D)$, s.t. $\langle \mathcal{A}^*_i,t^*_i\rangle \in W$, $m^* \in \mathbb{Z}^+$, and $\sum_{i=1}^{m^{*}} t^*_i \leq \mathcal{T}$.
\end{definition}

\subsection{AutoAE: Automatic AE Construction}
\label{sec:alg}
We now introduce AutoAE, a simple yet efficient approach for automatically constructing AEs.

As shown in Algorithm~\ref{alg:autoae}, it starts with an empty ensemble $\mathbb{A}$ (line 1) and repeatedly finds the attack $\mathcal{A}^*$ and its iteration steps $t^*$ that, if included in $\mathbb{A}$, maximizes ensemble improvement per additional iteration spent (line 3).
Here, ties are broken by choosing the attack with the fewest iteration steps.
After this, $\langle \mathcal{A}^*, t^* \rangle$ is subject to the following procedure:
if adding it to $\mathbb{A}$ does not improve the success rate of the latter (which means the algorithm has converged), or results in a larger number of used iterations steps than ${\mathcal{T}}$, AutoAE will terminate and return $\mathbb{A}$ (line 4);
%$\langle \mathcal{A}^*, t^* \rangle$ does not improve the success rate of $\mathcal{A}$, which means the algorithm has converged, AutoAE will immediately terminate and return $\mathcal{A}$ (line 4).
otherwise $\langle \mathcal{A}^*, t^* \rangle$ is appended to $\mathbb{A}$ (line 5).
%The above process iterates until the unused iteration steps $t_{unused}$ becomes to 0 (line 2 and line 6).
%The final output $\mathbb{A}$ can be further simplified by 
For the AE returned by AutoAE, we can further shrink (simplify) it without impairing its performance.
Specifically, we remove all pairs $\langle \mathcal{A},t \rangle \in \mathbb{A}$ satisfying $\exists \langle \mathcal{A}',t' \rangle \in \mathbb{A}$ such that $\mathcal{A}=\mathcal{A}' \land t\leq t'$.
This procedure can effectively reduce the computational complexity of $\mathbb{A}$, especially when $|\mathbb{S}|$ is small (see Section~\ref{exp_constructedAE}).

\begin{algorithm}[tbp]
	\LinesNumbered
	\KwIn{classifier $\mathcal{F}$, annotated training set $D$, candidate attack pool $\mathbb{S}$, maximum total iteration steps $\mathcal{T}$}
	\KwOut{$\mathbb{A}$}
	$\mathbb{A} \leftarrow []$, $t_{used} \leftarrow 0$;\\
	%	$W \leftarrow \{\langle \mathcal{A}, t \rangle|\mathcal{A}\in\mathbb{S} \land t\in \mathbb{Z}^+ \land t\leq \mathcal{T}\}$;\\
	\While{\textbf{true}}
	{		Let $\langle \mathcal{A}^*, t^* \rangle \leftarrow \argmax\limits_{\langle \mathcal{A}, t \rangle \in W} \frac{Q(\mathbb{A} \oplus \langle \mathcal{A}, t \rangle)-Q(\mathbb{A})}{t}$;\\
		%			\lIf{$Q(\mathbb{A} \oplus \langle \mathcal{A}^*, t^* \rangle) = Q(\mathbb{A})$}{return $\mathbb{A}$}
		\lIf{$Q(\mathbb{A} \oplus \langle \mathcal{A}^*, t^* \rangle) = Q(\mathbb{A})$ \textbf{or} $t_{used}+t^* > \mathcal{T}$}{return $\mathbb{A}$}
		$\mathbb{A} \leftarrow \mathbb{A} \oplus \langle \mathcal{A}^*, t^* \rangle $;\\
		$t_{used} \leftarrow t_{used}+t^*$;\\
	}
	\Return{$\mathbb{A}$}
	\caption{AutoAE}
	\label{alg:autoae}
\end{algorithm}

The main computational costs of AutoAE are composed of the function queries to $Q(\cdot)$ (line 3).
Specifically, each iteration of the algorithm will query $Q(\cdot)$ for $O(n{\mathcal{T}})$ times.
In the worst case, AutoAE will run for $\mathcal{T}$ iterations where in each iteration $t_{used}$ increases only by 1.
Hence, the total number of function queries consumed by AutoAE is $O(n{\mathcal{T}}^2)$.
In practice, we can discretize the range of iteration steps into a few uniform-spacing values to largely reduce the number of function queries;
furthermore, we can run the attacks in $\mathbb{S}$ on $D$ in advance to collect their performance data, which can be used to directly compute the value of $Q(\cdot)$ when it is being queried, instead of actually running the ensembles.

\subsection{Theoretical Analysis}
\label{sec:theory}
We now theoretically analyze the performance of AutoAE.
%For notational convenience, in below we temporarily omit $D$ in the function $Q(\cdot,D)$, and use $\mathbb{A} \oplus \mathbb{A}'$ to denote appending.
Let $V$ be the set of all possible AEs that can be constructed based on $W$.
Define $\ell(\mathbb{A})$ the total iteration steps consumed by $\mathbb{A}$, i.e., $\ell(\mathbb{A})=\sum_{\langle \mathcal{A}, t \rangle \in \mathbb{A}}t$.
Our analysis is based on the following fact.
\begin{fact}
\label{fact:sub}
$Q(\cdot)$ is a monotone submodular function, i.e., for any $\mathbb{A},\mathbb{A}' \in V$ and any $\langle \mathcal{A},t\rangle \in W$, it holds that $Q(\mathbb{A}) \leq Q(\mathbb{A} \oplus \mathbb{A}')$, and $Q\left(\mathbb{A} \oplus \mathbb{A}' \oplus\langle \mathcal{A},t\rangle\right)-Q\left(\mathbb{A} \oplus \mathbb{A}'\right) \leq Q\left(\mathbb{A} \oplus\langle \mathcal{A},t\rangle\right)-Q\left(\mathbb{A}\right)$.
\end{fact}
\begin{proof}
	By definition of $Q$, for any given instance $\pi$, we have
	$Q(\mathbb{A}\oplus \mathbb{A}', \pi) =\max_{\langle \mathcal{A},t \rangle \in  \mathbb{A}\oplus \mathbb{A}'} Q\left( \langle \mathcal{A},t \rangle, \pi \right) \geq \max_{\langle \mathcal{A},t \rangle \in  \mathbb{A}} Q\left( \langle \mathcal{A},t \rangle, \pi \right)= Q(\mathbb{A},\pi)$. Thus the monotonicity holds.
	
	To prove submodulariy, we first notice that for any given instance $\pi$, if $Q\left(\mathbb{A} \oplus\langle \mathcal{A},t\rangle, \pi\right)-Q\left(\mathbb{A}, \pi\right) = 1$, then $Q\left(\mathbb{A} \oplus \mathbb{A}' \oplus\langle \mathcal{A},t\rangle\right)-Q\left(\mathbb{A} \oplus \mathbb{A}'\right) \leq 1$;
	on the other hand, if $Q\left(\mathbb{A} \oplus\langle \mathcal{A},t\rangle, \pi\right)-Q\left(\mathbb{A}, \pi\right) = 0$, then it must hold that $Q\left(\mathbb{A} \oplus \mathbb{A}' \oplus\langle \mathcal{A},t\rangle\right)-Q\left(\mathbb{A} \oplus \mathbb{A}'\right) = 0$.
	In either case, $Q\left(\mathbb{A} \oplus\langle \mathcal{A},t\rangle, \pi\right)-Q\left(\mathbb{A}, \pi\right) \geq Q\left(\mathbb{A} \oplus \mathbb{A}' \oplus\langle \mathcal{A},t\rangle\right)-Q\left(\mathbb{A} \oplus \mathbb{A}'\right)$. Thus the submodularity holds.
\end{proof}
Intuitively, $Q$ exhibits a so-called diminishing returns property \citep{qian2015subset} that the marginal gain of adding $\langle \mathcal{A},t\rangle$ diminishes as the ensemble size increases.
Based on Fact~\ref{fact:sub}, we can establish the following useful lemma that states in terms of the increase in $Q$ per additional iteration step spent, the results from appending an AE to the current ensemble are always upper bounded by the results from appending the best $\langle \mathcal{A},t\rangle$ pair to it.
\begin{lemma}
\label{lem:incre}
For any $\mathbb{A},\mathbb{A}' \in V$, it holds that
\begin{equation*}
	\frac{Q\left(\mathbb{A} \oplus \mathbb{A}'\right)-Q\left(\mathbb{A}\right)}{\ell(\mathbb{A}')} \leq \max _{\langle\mathcal{A},t \rangle}\left\{\frac{Q\left(\mathbb{A} \oplus\langle\mathcal{A},t \rangle\right)-Q\left(\mathbb{A}\right)}{t}\right\}.
\end{equation*}
\end{lemma}
\begin{proof}
	Let $r$ denote the right hand side of the inequality. Let $\mathbb{A} = [ a_1,a_2,\dots,a_L ]$, where $a_l=\langle \mathcal{A}_l,t_l \rangle$. Let $\Delta_l = Q(\mathbb{A} \oplus [a_1,a_2,\dots,a_l])- Q(\mathbb{A} \oplus [a_1,a_2,\dots,a_{l-1}])$.
	We have
	\begin{align*}
		&Q(\mathbb{A} \oplus \mathbb{A}') =Q(\mathbb{A})+\sum\limits_{l=1}^{L}\Delta_{l}     \qquad \text{(telescoping series)}\\
		&\leq Q(\mathbb{A})+\sum\limits_{l=1}^{L}(Q(\mathbb{A} \oplus a_l)-Q(\mathbb{A}))\qquad \text{(submodularity)}\\
		&\leq Q(\mathbb{A})+\sum\limits_{l=1}^{L}r\cdot\tau_l \qquad \text{(definition of $r$ )}\\
		&=Q(\mathbb{A})+r\cdot \ell(\mathbb{A}').
	\end{align*}
	Rearranging this inequality gives $\frac{Q(\mathbb{A}\oplus \mathbb{A}')-Q(\mathbb{A})}{\ell(\mathbb{A}')}\leq r$, as claimed.
\end{proof}
We now present the approximation bounds of AutoAE on maximizing $Q(\cdot)$ in Eq.~(\ref{eq:Q}) and minimizing $C(\cdot)$ in Eq.~(\ref{eq:C}), in Theorem~\ref{the:Q} and Theorem~\ref{the:C}, respectively.
%for maximizing $Q(\cdot)$ in Eq.~(\ref{eq:Q}) and minimizing $C(\cdot)$ in Eq.~(\ref{eq:C}), respectively.
\begin{theorem}
\label{the:Q}
Let $\mathbb{A}=[\langle \mathcal{A}_1,t_1\rangle,\langle \mathcal{A}_2,t_2\rangle,...,\langle \mathcal{A}_m,t_m\rangle]$ be the AE returned by AutoAE, and let $T = \ell(\mathbb{A})$.
It holds that $Q(\mathbb{A}) \geq (1-\frac{1}{e})\max_{\mathbb{A}'\in V}Q(\mathbb{A}'(T))$, where $\mathbb{A}'(T)$ is the AE resulting from truncating $\mathbb{A}'$ at iteration steps $T$.
In other words, $\mathbb{A}$ achieves $(1-1/e)$-approximation for the optimal success rate achievable at any given iteration steps $t\leq T$. 
\end{theorem}
\begin{proof}
Let $Q^{*} = \max_{\mathbb{A}'\in V}Q(\mathbb{A}'(T))$.
For integer $j\in [1,m]$, define $\mathbb{A}_j=[\langle \mathcal{A}_1,t_1\rangle,...,\langle \mathcal{A}_{j-1},t_{j-1}\rangle]$, i.e., the AE at the beginning of the $j$-iteration of AutoAE.
Define $Q_j=Q(\mathbb{A}_j)$ (note $Q_1=0$) and $\Delta_j=Q^*-Q_j$.
%i.e., the success rate achieved by the first $j-1$ attacks in $\mathbb{A}$, and $\Delta_j=Q^*-Q_j$.
Then let $s_j=\max_{\langle \mathcal{A}, t \rangle \in W} \frac{Q(\mathbb{A}_j \oplus \langle \mathcal{A}, t \rangle)-Q_j}{t}$, i.e., the maximum value obtained at the $j$-iteration in line 3 of Algorithm~\ref{alg:autoae}.
It is straightforward to verify that $s_j=(\Delta_j-\Delta_{j+1})/t_j$.

We first prove for any $\mathbb{A}'\in V$ and $t \leq \mathcal{T}$, it holds that
\begin{equation}
\label{eq:key}
Q(\mathbb{A}'(t)) \leq Q_{j} +t \cdot s_{j}.
\end{equation}
By monotonicity in Fact~\ref{fact:sub}, we have $Q(\mathbb{A}'(t)) \leq Q(\mathbb{A}_j \oplus \mathbb{A}'(t))$, while by Lemma~\ref{lem:incre}, it further holds that $[Q(\mathbb{A}_j \oplus \mathbb{A}'(t)) - Q_j]/t \leq s_i$.
The proof is complete.

Since $T \leq \mathcal{T}$, by Eq.~(\ref{eq:key}), we have $Q^*\leq Q_j +T \cdot s_{j}$, which implies $\Delta_j \leq T \cdot s_{j} = T\cdot [(\Delta_j-\Delta_{j+1})/t_j]$.
Rearranging this inequality gives $\Delta_{j+1} \leq \Delta_{j}(1-t_j/T)$.
Unrolling it, we have $\Delta_{m+1} \leq \Delta_1(\prod_{j=1}^{m}1-t_j/T)$.
The product series is maximized when $t_j = T/m$ for all $j$, since $\sum_{j=1}^{m}t_j = T$.
Thus it holds that
\begin{equation*}
Q^*-G_{m+1} = \Delta_{m+1}\leq \Delta_1(1-\frac{1}{m})^m < \Delta_1\frac{1}{e}\leq Q^*\frac{1}{e}.
\end{equation*}
Thus $G_{m+1}=Q(\mathbb{A}) \geq (1-\frac{1}{e})Q^*$, as claimed.
\end{proof}

\begin{theorem}
\label{the:C}
When $\mathcal{T} \rightarrow \infty$, $C(\mathbb{A})\leq  4 \min_{\mathbb{A}' \in V}C(\mathbb{A}')$, i.e., $\mathbb{A}$ achieves $4$-approximation for the optimal expected iteration steps to achieve the optimal success rate. 
\end{theorem}
%\begin{proof}
%Define $R_j = 1-Q_j$, $a_j = \lfloor{\frac{R_j}{2s_j}}\rfloor$ and $b_j=\frac{R_j}{2}$, where $R_j$ is the current distance to the optimal success rate at the beginning of the $j$-iteration of AutoAE, and $a_j$ indicates the iterations steps needed by AutoAE to achieve it according to $s_j$.
%By Eq.~(\ref{eq:key}), we have
%$\max_{\mathbb{A}' \in V}Q(\mathbb{A}'(a_j)) \leq Q_j+a_j s_j \leq Q_j+R_j/2.$
%Define $h(t)=1-\max_{\mathbb{A}' \in V}Q(\mathbb{A}'(t))$, i.e., the best possible success rate achievable at iteration steps $t$.
%Then we have  $h(a_j) \geq R_j - R_j/2=b_j$.
%By monotonicity in Fact~\ref{fact:sub}, it holds that $h(t)$ is non-increasing and $y_j$ is also non-decreasing, which implies 
%$\sum_{t=1}h(t) \geq \sum_{j=1}\frac{R_j}{2s_j}(b_j-b_{j+1})$.
%Thus we have
%\begin{equation*}
%\small
%\min_{\mathbb{A}' \in V}C(\mathbb{A}') \geq \sum_{t=1}h(t) \geq \frac{1}{4}\sum\limits_{j \geq 1}R_j\frac{(R_j-R_{j+1})}{s_j}\geq\frac{1}{4}C(\mathbb{A}).
%%=\frac{1}{4}\sum \limits_{j \geq 1}R_jt_j .
%\end{equation*}
%The last step is by monotonicity. The proof is complete.
%\end{proof}
The proof is similar to \citet[Theorem~7]{streeter2007online} and is omitted here.
The above Theorem~\ref{the:Q} and Theorem~\ref{the:C} indicate that the AEs constructed by AutoAE have performance guarantees in both quality (success rate) and complexity (iteration steps).

\NewDocumentCommand{\anote}{}{\makebox[0pt][l]{$^*$}}

\begin{table*}[tbp]
	\centering
	\scalebox{0.95}{
		\begin{tabular}{rlccccc}
			\toprule
			{\#} & \multicolumn{1}{p{20em}}{{Paper}} & {Clean} & {AutoAttack} & CAA   & {Best Known} & {AutoAE} \\
			\midrule
			\multicolumn{7}{l}{{CIFAR-10 - $l_{\infty}$ - \  $\epsilon=8/255$}} \\
			\midrule
			{1} & \citet{rebuffi2021fixing}     & {92.23} & {66.58 } & —     & \underline{66.56} & \textbf{66.53}\anote  \\
			{2} & \citet{gowal2021uncovering}     & {91.10} & {65.88 } & 65.87  &  65.87 & \textbf{65.83}\anote   \\
			{3} & \citet{rebuffi2021fixing}     & {88.50} & {64.64 } & —     &   \underline{64.58} & \textbf{64.58}   \\
			{4} & \citet{rebuffi2021fixing}     & {88.54} & {64.25 } & —     &  \underline{64.20} & \textbf{64.22}   \\
			{5} & \citet{rade2021helperbased}     & {91.47} & {62.83 } & —     &  62.83 & \textbf{62.82}\anote   \\
			{6} & \citet{gowal2021uncovering}     & {89.48} & {62.80 } & \textbf{62.77}  &  \underline{62.76} &  62.79   \\
			{7} & \citet{rade2021helperbased}     & {88.16} & {60.97 } & —     &  60.97 & \textbf{60.88}\anote   \\
			{8} & \citet{rebuffi2021fixing}     & {87.33} & {60.75 } & —     &  \underline{60.73} & \textbf{60.71}\anote   \\
			{9} & \citet{sridhar2021improving}     & {86.53} & {60.41 } & —     &  60.41 & \textbf{60.32}\anote   \\
			{10} & \citet{wu2020adversarial}    & {88.25} & {60.04 } & 60.04  &  60.04 & \textbf{60.00}\anote   \\
			{11} & \citet{sridhar2021improving}    & {89.46} & {59.66 } & —     &   59.66 & \textbf{59.58}\anote   \\
			{12} & \citet{zhang2021geometryaware}    & {89.36} & {59.64 } & —     &   59.64 & \textbf{59.18}\anote   \\
			{13} & \citet{CarmonRSDL19}    & {89.69} & {59.53 } &  \textbf{59.38}    &   59.38 & \textbf{59.38}   \\
			{14} & \citet{sehwag2021robust}    & {85.85} & {59.09 } & —     &   59.09 & \textbf{59.05}\anote   \\
			{15} & \citet{rade2021helperbased}    & {89.02} & {57.67 } & —     & 57.67 &  \textbf{57.63}\anote   \\
			{16} & \citet{gowal2021uncovering}    & {85.64} & {56.86 } & 56.83  &  \underline{56.82} &  \textbf{56.82}   \\
			{17} & \citet{cifar10-linf-17-Wang2020Improving}    & {87.50} & {56.29 } & 56.29  &  56.29 & \textbf{56.20}\anote   \\
			{18} & \citet{wu2020adversarial}    & {85.36} & {56.17 } & 56.15  &   56.15 & \textbf{56.11}\anote   \\
			{19} & \citet{cifar10-linf-19-pang2021bag}    & {86.43} & {54.39 } & 54.26  &  54.26 & \textbf{54.22}\anote   \\
			{20} & \citet{cifar10-linf-20-pang2020boosting}    & {85.14} & {53.74 } & 53.75  &  53.74 & \textbf{53.68}\anote   \\
			\midrule
			\multicolumn{7}{l}{{CIFAR-10 - $l_{2}$ - \  $\epsilon=0.5$}} \\
			\midrule
			{1} & \citet{rebuffi2021fixing}     & {95.74 } & {82.32 } & —    & 82.32 &  \textbf{82.31}\anote  \\
			{2} & \citet{gowal2021uncovering}     & {94.74 } & {80.53 } & 80.47  & 80.47 &  \textbf{80.43}\anote   \\
			{3} & \citet{rebuffi2021fixing}     & {92.41 } & {\textbf{80.42} } & —   & 80.42 &  \textbf{80.42}  \\
			{4} & \citet{rebuffi2021fixing}     & {91.79 } & {\textbf{78.80} } & —   & 78.80 &  \textbf{78.80}  \\
			{5} & \citet{augustin2020adversarial}     & {93.96 } & {78.79 } & —     &  78.79 & \textbf{78.79}  \\
			{6} & \citet{augustin2020adversarial}     & {92.23 } & {\textbf{76.25} } & —     &  76.25 & \textbf{76.25}  \\
			{7} & \citet{rade2021helperbased}     & {90.57 } & {\textbf{76.15} } & —     &  76.15 & \textbf{76.15}  \\
			{8} & \citet{sehwag2021robust}     & {90.31 } & {\textbf{76.12} } & —     &  76.12 & \textbf{76.12}  \\
			{9} & \citet{rebuffi2021fixing}     & {90.33 } & {\textbf{75.86} } & —     &  75.86 & \textbf{75.86}  \\
			{10} & \citet{gowal2021uncovering}    & {90.90 } & {\textbf{74.50} } & \textbf{74.50}  &  74.50  & \textbf{74.50}   \\
			{11} & \citet{wu2020adversarial}    & {88.51 } & {\textbf{73.66} } & \textbf{73.66}  &  73.66  & \textbf{73.66}   \\
			{12} & \citet{sehwag2021robust}    & {89.52 } & {73.39 } & —     &  73.39  & \textbf{73.38}\anote   \\
			{13} & \citet{augustin2020adversarial}    & {91.08 } & {\textbf{72.91} } & 72.93  &  72.91  & \textbf{72.91}   \\
			{14} & \citet{cifar10-l2-14-robustness}    & {90.83 } & {\textbf{69.24} } & —   & 69.24 &  \textbf{69.24}  \\
			{15} & \citet{RiceWK20}    & {88.67 } & {67.68 } & —  & 67.68 &  \textbf{64.40}\anote   \\
			\midrule
			\multicolumn{7}{l}{{CIFAR-100 - $l_{\infty}$ - \  $\epsilon=8/255$}} \\
			\midrule
			{1} & \citet{gowal2021uncovering}     & {69.15 } & {36.88 } & —   & 36.88 &  \textbf{36.86}\anote  \\
			{2} & \citet{rebuffi2021fixing}     & {63.56 } & {34.64 } & —  & 34.64 &  \textbf{34.62}\anote  \\
			{3} & \citet{rebuffi2021fixing}     & {62.41 } & {32.06 } & —   & 32.06 &  \textbf{32.02}\anote  \\
			{4} & \citet{cifar100-linf-4-cui2021learnable}     & {62.55 } & {30.20 } & —   & 30.20 &  \textbf{30.13}\anote  \\
			{5} & \citet{gowal2021uncovering}     & {60.86 } & {30.03 } & —  & 30.03  &  \textbf{30.00}\anote  \\
			
			\midrule
			\multicolumn{7}{l}{{ImageNet - $l_{\infty}$ - \  $\epsilon=4/255$}} \\
			\midrule
			{1} & \citet{salman2020adversarially}     & {68.46 } & {38.14 } & —  & 38.14 &  \textbf{37.96}\anote  \\
			{2} & \citet{salman2020adversarially}     & {64.02 } & {34.96 } & —   & 34.96 &  \textbf{34.36}\anote  \\
			{3} & \citet{cifar10-l2-14-robustness}     & {62.56 } & {29.22 } & —   & 29.22 &  \textbf{28.86}\anote  \\
			{4} & \citet{imagenet-linf-4-wong2020fast}     & {55.62 } & {26.24 } & —  & 26.24 &  \textbf{24.80}\anote   \\
			{5} & \citet{salman2020adversarially}     & {52.92 } & {25.32 } & —   & 25.32 &  \textbf{25.00}\anote  \\
			\bottomrule
	\end{tabular}}
	\caption{Robustness evaluation of the top defenses listed on RobustBench Leaderboard, by AutoAttack, CAA (if available) and AutoAE, in terms of test robust accuracy (\%).
	The clean test accuracy and the best known robust accuracy are also listed.
	For each defense model, the strongest AE is indicated in \textbf{bold}.
	A best known robust accuracy is underlined (``\_'') if it is not obtained by either AutoAttack or CAA, but some other attacks.
	The robust accuracy of AutoAE is marked with ``*''  if it provides lower (better) robust accuracy than the best known one.
	%		The difference between the best known robust accuracy and the one of AutoAE is also computed; if negative (in red), AutoAE provides lower (better) robust accuracy. 
	Note that CAA builds an attack policy for each defense separately, while for AutoAttack and AutoAE the AE remains the same under each attack setting.
}
	\label{tab:main_results}
\end{table*}

\section{Experiments}
\label{sec:exp}
The main goal of the experiments is to evaluate the potential of AutoAE as a standard robustness evaluation approach.
In particular, when using AutoAE to construct AEs, \textit{we are not providing it with the currently strongest defense models, but some seemingly outdated ones}.
The purpose of such a setting is to assess whether AutoAE can transfer well from weaker defenses to previously unseen and stronger ones.
This is an important characteristic for reliable robustness evaluation.

We evaluate the adversarial robustness with the $l_{\infty}$ and $l_{2}$ distance constraints of 45 top defenses on the leaderboard of RobustBench \citep{croce2021robustbench} (up to Jan 2022).
Note all these defenses have a fully deterministic forward pass; thus the results of running our AEs on them are stable (i.e., no randomness).
%The scripts and codes for repeating our experiments are available in the supplementary.

\subsection{Constructing AEs with AutoAE}
\label{exp_constructAE}
The candidate attack pool consists of 11 recently proposed individual attacks, including the $l_{\infty}$ and the $l_{2}$ versions of $\text{APGD}_\text{CE}$ (CE), $\text{APGD}_\text{DLR}$ (DLR), $\text{FAB}$, Carlini \& Wagner Attack (CW) \citep{Carlini017} and MultiTargeted Attack (MT) \citep{Gowal2019}, as well as the $l_2$ attack Decoupled Direction and Norm (DDN) \citep{RonyHOASG19}.
We use the implementations of APGD attack and FAB attack from the code repository of AutoAttack.
For MT attack, CW attack, and DDN attack, we use the implementations from the repository of CAA.
%We use their open source implementations without any tuning;
Following AutoAttack \citep{Croce020a}, for attacks that require step size parameter, we use the adaptive versions of them that are step size free.
We allow each attack with only one restart, and set the AE's maximum total iteration steps to 1000, forcing AutoAE to construct AEs with low computational complexity (see Figure~\ref{fig:comparison}).
Besides, we discretize the range of iteration steps of these attacks into 8 uniform-spacing values to reduce the computational costs of AutoAE (see Section~\ref{sec:alg}).
% Details of all these attacks and their discretized values can be found in Appendix~C.
Finally, we randomly select 5,000 instances from the training set of CIFAR-10 as the annotated training set, and use two CIFAR-10 adversarial training models from Robustness library \citep{cifar10-l2-14-robustness} to construct the AE for $l_{\infty}$ and $l_2$ attacks, respectively.
Note these two defense models only ranked 49 and 14 on their corresponding leaderbords.

\subsection{Baselines and Defense Models}
We compare AutoAE with the other AEs, i.e., AutoAttack and CAA.
The former is constructed by human experts while the latter is constructed automatically (like ours).
Both of them can achieve generally better robustness evaluation than individual attacks.
With Robustbench as the testbed, we use the top 15 defenses on the CIFAR-10 ($l_{\infty},\epsilon=8/255$) and the CIFAR-10 ($l_{2},\epsilon=0.5$) leaderboards, and the top 5 defenses on the CIFAR-100 ($l_{\infty},\epsilon=8/255$) and the ImageNet ($l_{\infty},\epsilon=4/255$) learderboards.
Following the guidelines of RobustBench, we use its integrated interface to test our AEs, where for CIFAR-10 and CIRAR-100 all the 10,000 test instances are used, while for ImageNet a fixed set of 5,000 test instances are used.
% (we also test our AEs on the entire ImageNet test set, see Appendix~D for results).
Note the parameter $\epsilon$ of the attacks in our AEs are always set in line with the defense being attacked.
Since the leaderboards also list the robust accuracy assessed by AutoAttack and the best known accuracy, we directly compare them with the results achieved by our AEs.
Besides, CAA \citep{MaoCWSHX21} has been shown to achieve its best performance when building a separate attack policy for each defense model, we collect such results for comparison, leading to five more defenses used in our experiments.
Note for some defenses, CAA has not been applied to them, in which case its robust accuracy is unavailable.
%Overall, the experiments involve in total 45 defenses, where the performance of CAA on certain defenses are unavailable, because for these defenses CAA has not searched the policy

\begin{figure}[tbp]
	\centering
	\begin{subfigure}[b]{1.0\columnwidth}
		\centering
		\scalebox{1.0}{\includegraphics[width=\linewidth]{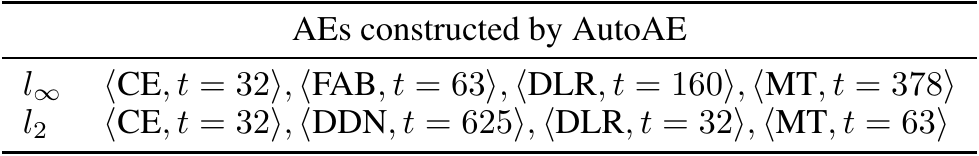}}
		\label{fig:AEs}
	\end{subfigure}
	\hfill
	\begin{subfigure}[b]{0.49\columnwidth}
		\centering
		\scalebox{1.0}{\includegraphics[width=\linewidth]{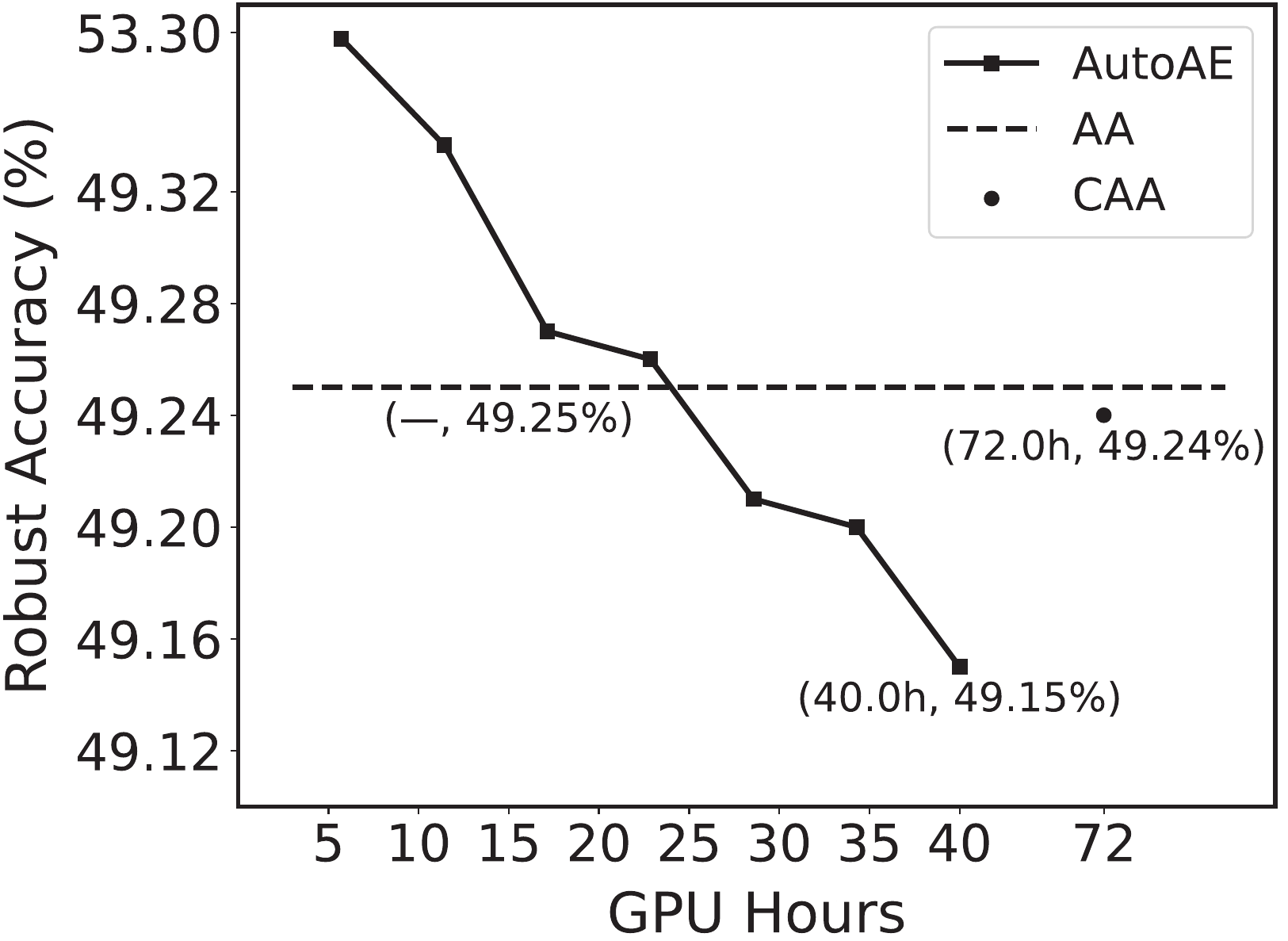}}
		%		\caption{$l_{\infty}$}
	\end{subfigure}
	\hfill
	\begin{subfigure}[b]{0.49\columnwidth}
		\centering
		\scalebox{1.0}{\includegraphics[width=\linewidth]{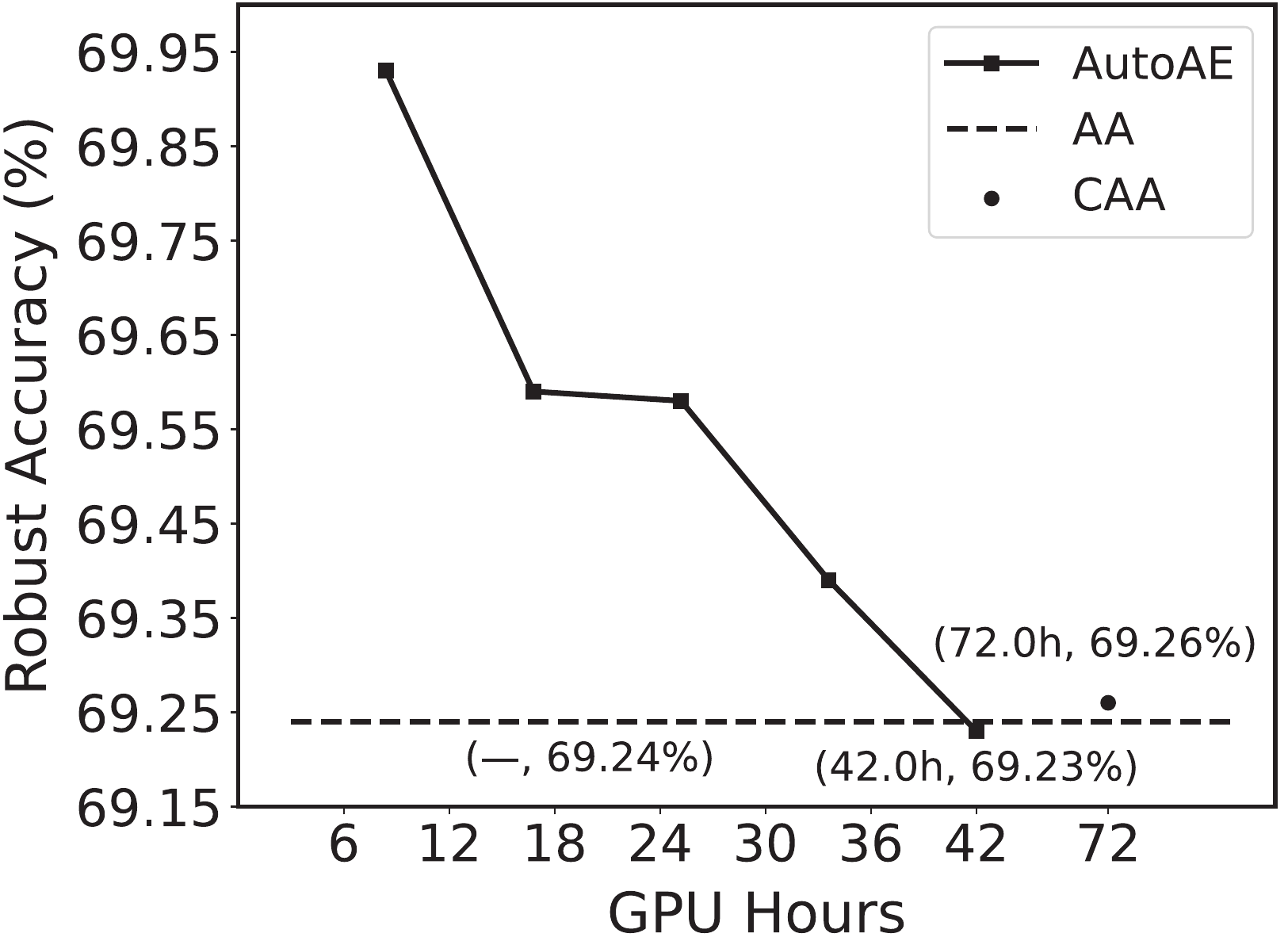}}
		%		\caption{$l_2$}
	\end{subfigure}
	\caption{Visualization of the two constructed AEs (\textbf{upper}), and their \textit{test} robust accuracy progress after each iteration of AutoAE (\textbf{lower}), on attacking $l_{\infty}$ (\textbf{left}) and $l_{2}$ (\textbf{right}) CIFAR-10 adversarial training models.
		The results of AutoAttack (a horizontal line indicating its test robust accuracy) and CAA (a point indicating its test robust accuracy and GPU hours used to search for the attack policy) are also illustrated.}
	\label{fig:per_progress}
\end{figure}

\subsection{Analysis of the Constructed AEs}
\label{exp_constructedAE}
Figure~\ref{fig:per_progress} presents the two constructed AEs and the progress of their \textit{test} robust accuracy along the iterations of AutoAE.
The results of AutoAttack and CAA are also illustrated.
For visualization, the GPU hours consumed by AutoAE to collect the candidate attacks' performance data are distributed evenly over all iterations.
The first observation from Figure~\ref{fig:per_progress} is that both our AEs outperform the two baselines, and their test performance improves monotonically from one iteration to the next, which accords to the monotonicity in Fact~\ref{fact:sub}.
Actually, AutoAE iterates for 7 times and 5 times in $l_\infty$ and $l_2$ scenarios, respectively, and then shrink (see Section~\ref{sec:alg}) the sizes of the AEs to 4 and 3, respectively.

In terms of efficiency, AutoAE generally only needs to consume around half GPU hours of CAA, making the former the more efficient approach for AE construction.
One may observe that for both AEs, AutoAE chooses the three strong attacks (CE, DLR, and MT) that use different losses, which implies a combination of diverse attacks is preferred by it. 
Finally, we compare the AE for $l_\infty$ attack with the recently proposed $l_\infty$ attacks and illustrate the results in Figure~\ref{fig:comparison}, where among all the attacks, our AE achieves the best robustness evaluation with low computational complexity.

\subsection{Results and Analysis}
Table~\ref{tab:main_results} presents the main results.
Overall, AutoAE can provide consistently better robustness evaluation than the two baselines.
In all except one cases AutoAE achieves the best robustness evaluation among the AEs, which is far more than AutoAttack with 10 best ones.
Compared to CAA, in 9 out of 13 cases where it has been applied, AutoAE achieves better robustness evaluation, and in three out of the left four cases, AutoAE is as good as CAA.
Recalling that CAA builds an attack policy for each defense separately while AutoAE builds a unified AE for different defenses, it is conceivable that the performance advantage of AutoAE would be greater if using it to build a separate AE for each defense.

Notably, in 29 cases AutoAE achieves better robustness evaluation than the best known one.
Among them many defenses are ranked very highly on the leaderboard, demonstrating that our AEs can transfer well from the weaker defenses (i.e., the defenses used in the construction process) to newer and stronger ones.
In summary, the strong performance of AutoAE proves itself a reliable evaluation protocol for adversarial robustness, and also indicates the huge potential of automatic construction of AEs.

\subsection{Attack Transferability}
An evaluation approach for adversarial robustness is more desirable if it exhibits good attack transferability.
%It is interesting to investigate the attack transferability of our approach.
We choose three recent CIFAR-10 defenses ($l_{\infty},\epsilon=8/255$) Semi-Adv \citep{CarmonRSDL19}, LBGAT \citep{cifar100-linf-4-cui2021learnable}, and MMA \citep{DingSLH20}, to which both AutoAttack and CAA have been applied, and transfer the adversarial examples (generated based on the 10,000 test instances) between them.
As presented in Table~\ref{tab:transfer}, AutoAE consistently demonstrate the best transferability in all the transfer scenarios.

\begin{table}[tbp]
	\centering
	\scalebox{0.86}{
		\begin{tabular}{lcccc}
			\toprule
			Transfer & clean & AutoAttack    & CAA   & AutoAE \\
			\midrule
			Semi-Adv → LBGAT & 88.22 & 71.95 & 74.72 & \textbf{70.78} \\
			Semi-Adv → MMA   & 84.36 & 71.79 & 72.96 & \textbf{70.61} \\
			LBGAT → Semi-Adv & 89.69 & 76.13 & 80.23 & \textbf{74.93} \\
			LBGAT → MMA      & 84.36 & 70.11 & 71.78 & \textbf{68.53} \\
			MMA → Semi-Adv   & 89.69 & 83.58 & 84.99 & \textbf{82.52} \\
			MMA → LBGAT      & 88.22 & 80.00 & 82.28 & \textbf{78.78} \\
			\bottomrule
	\end{tabular}}
	\caption{Evaluation results on attack transferability, in terms of test accuracy (\%). ``A → B'' means transferring adversarial examples generated based on model A to model B. The clean accuracy of target model is also listed. For each transfer scenario, the best value is indicated in bold.}
	\label{tab:transfer}
\end{table}

\subsection{Constructing AEs for Specific Defenses}
Although AutoAE is proposed for general robustness evaluation, it also enables researchers to easily build their own AEs to meet specific demands.
For example, one can use AutoAE to build powerful AEs targeted at some specific defense model.
Here, we provide AutoAE with the top-1 CIFAR-10 ($l_{\infty},\epsilon=8/255$) defense model in Table~\ref{tab:main_results} \citep{rebuffi2021fixing} (the candidate attack pool and the annotated dataset are the same as before), and the constructed AE obtains a test robust accuracy of 46.3\% on this defense.
% (see Appendix~E for details).
Recalling that in Table~\ref{tab:main_results} the corresponding robust accuracy is 66.53\%, such improvement implies that AutoAE can achieve even better robustness evaluation when given newer and stronger defenses.

\section{Discussion and Conclusion}
\label{sec:dis}
This work focuses on developing an easy-to-use approach for AE construction that can provide reliable robustness evaluation and also can significantly reduce the efforts of human experts.
To achieve this goal, we have proposed AutoAE, a conceptually simple approach with theoretical performance guarantees in both AE quality and complexity.

Compared with existing AE construction approach, AutoAE mainly has three advantages.
First, it has no hard-to-set parameters.
Second, it consumes fewer computation resources (measured by GPU hours) to construct high-quality AEs.
Third, it could construct AEs that transfer well from weaker defenses to previously unseen and stronger ones.
Also, compared to the fixed AE built by human experts, AutoAE enables researchers to conveniently build their own AEs to meet specific demands, e.g., achieving strong attack performance on specific defenses or incorporating the newly-proposed attacks into AEs to keep up with the field.

To conclude, let us discuss some of the limitations of our approach and possible directions for future work.
First, since AutoAE builds an AE based on candidate attacks and a defense model, it may perform not well when the candidate attacks are homogeneous such that the complementarity among them is limited, or the defense is too weak to discriminate between good and bad AEs.
Thus, it is possible to include more diverse attacks in the candidate pool to build even better AEs.
Second, currently AutoAE only involves appending a $\langle \mathcal{A},t\rangle$ pair to the ensemble, which may result in stagnation into the local optimum, due to the greedy nature of the approach.
It is natural to integrate more pair operations such as deletion of a pair and exchange of two pairs into AutoAE to help it escape from the local optimum.
Third,  as a generic framework, AutoAE can be used to construct AEs for the black-box and even decision-based settings. Also, it can be used to construct AEs in other domains such as natural language processing and speech.
Lastly, from the general perspective of problem solving, AEs fall into the umbrella of algorithm portfolio (AP) \citep{LiuT019,TangLYY21,LiuTY22}, which seeks to combine the advantages of different algorithms.
Hence, it is valuable to study how to apply the approach proposed in this work to the automatic construction of APs, and vice versa.
 
%Lastly, as a generic framework, AutoAE can be used to craft AEs for the black-box and even decision-based settings. Also, it can be used to construct AEs in other domains such as natural language processing and speech. It is an interesting direction for future to study such choices.

%Although this work is oriented towards the field of robustness evaluation, these results may also be of independent interest to algorithm-portfolio scheduling \citep{LiuT019,BalcanSV21}.

% \section*{Acknowledgments}
% This work was supported in part by the Program for Guangdong Introducing Innovative and Entrepreneurial Teams under Grant 2017ZT07X386, in part by the Shenzhen Peacock Plan under Grant KQTD2016112514355531, and in part by the Guangdong Provincial Key Laboratory under Grant 2020B121201001.

\appendix

\section{NP-Hardness of the AE Construction Problem}
The subset selection problem with general cost constraints \citep{NemhauserW78} is presented below.
\begin{definition}
Given a monotone submodular objective function $f:2^V \rightarrow R^+$, a monotone cost function $c:2^V \rightarrow R^+$, the goal is to find $X^* \subseteq V$, such that:
\begin{equation}
\label{eq:subset}
    X^* = \argmax_{X \subseteq V} f(X) \ \ s.t.\ \ c(X)\leq B.
\end{equation}
\end{definition}
The AE construction problem with the only objective of maximizing $Q$, can be transformed to the above form.
Specifically, the objective function $f$ corresponds to $Q$, 
the ground set $V$ corresponds to $W=\{\langle \mathcal{A}, t \rangle|\mathcal{A}\in\mathbb{S} \land t\in \mathbb{Z}^+ \land t\leq \mathcal{T}\}$,
the cost function $c$ for a pair $\langle \mathcal{A}, t \rangle$ is defined as $c(\langle \mathcal{A}, t \rangle)=t$,
and the budget $B$ is the maximum total iteration steps $\mathcal{T}$.

It has been shown in \citep{QianSYT17} that  for any $\epsilon>0$, achieving a $(1-\frac{1}{e}+\epsilon)$ approximation ratio for the above optimization problem is NP-hard.
Hence, the AE construction problem, which is at least as hard as the above problem, is also NP-hard.

\section{Implementation Details of the Candidate Attacks}
In the experiments, we use the implementations of APGD attack and FAB attack from the  AutoAttack repository \footnote{\url{https://github.com/fra31/auto-attack}}. 
For MT attack, CW attack and DDN attack, we use their implementations from the repository of CAA \footnote{\url{https://github.com/vtddggg/CAA}}.
We discretize the range of iteration steps of these attacks into 8 uniform-spacing values to reduce the computational
costs of AutoAE, as shown in Table~\ref{tab:discrete}.

\begin{table}[h]
	\centering
	\begin{tabular}{lc}
		\toprule
		\multicolumn{1}{c}{Attack} & Iteration Steps vector\\
		\midrule
		MT    &  $ 63 \times [1, 2, \cdots,8] $\\
		CW    &  $ 125 \times [1, 2, \cdots,8] $ \\
		CE    &  $ 32 \times [1, 2, \cdots,8] $ \\
		DLR   &  $ 32 \times [1, 2, \cdots,8] $\\
		FAB   &  $ 63 \times [1, 2, \cdots,8] $\\
		DDN   &  $ 63 \times [1, 2, \cdots,8] $\\
		\bottomrule
	\end{tabular}%
	\caption{Discretized values of the iteration steps.}
	\label{tab:discrete}%
\end{table}%

For all these attacks, the default values of their hyper-parameters are used.

\section{Results on the Entire ImageNet Test Set}
We also test the AE constructed for $l_{\infty}$ attack on the entire ImageNet test set.
The results are presented in Table~\ref{tab:imagent}.
	
\section{Constructing AEs for Specific Defenses}

\begin{table}[tbp]
	\centering
	\begin{tabular}{rlcccc}
		\toprule
		\textcolor[rgb]{ .039,  .039,  .039}{\#} & \multicolumn{1}{p{11.125em}}{\textcolor[rgb]{ .039,  .039,  .039}{Paper}} & \textcolor[rgb]{ .039,  .039,  .039}{Clean} & \cellcolor[rgb]{ .922,  .945,  .871} \textcolor[rgb]{ .039,  .039,  .039}{AutoAE} \\
		\midrule
		\multicolumn{4}{l}{\textcolor[rgb]{ .039,  .039,  .039}{ImageNet - $l_{\infty} - \epsilon =4/255$}} \\
		\midrule
		\textcolor[rgb]{ .039,  .039,  .039}{1} & \citet{salman2020adversarially}      & \textcolor[rgb]{ .039,  .039,  .039}{68.46 } & \cellcolor[rgb]{ .922,  .945,  .871} \textcolor[rgb]{ .039,  .039,  .039}{\textbf{37.5140 }} \\
		\textcolor[rgb]{ .039,  .039,  .039}{2} & \citet{salman2020adversarially}     & \textcolor[rgb]{ .039,  .039,  .039}{64.02 } & \cellcolor[rgb]{ .922,  .945,  .871} \textcolor[rgb]{ .039,  .039,  .039}{\textbf{34.5120 }} \\
		\textcolor[rgb]{ .039,  .039,  .039}{3} &  \citet{cifar10-l2-14-robustness}     & \textcolor[rgb]{ .039,  .039,  .039}{62.56 } & \cellcolor[rgb]{ .922,  .945,  .871} \textcolor[rgb]{ .039,  .039,  .039}{\textbf{28.6500 }} \\
		\textcolor[rgb]{ .039,  .039,  .039}{4} & \citet{imagenet-linf-4-wong2020fast}      & \textcolor[rgb]{ .039,  .039,  .039}{55.62 } & \cellcolor[rgb]{ .922,  .945,  .871} \textcolor[rgb]{ .039,  .039,  .039}{\textbf{24.4220 }} \\
		\textcolor[rgb]{ .039,  .039,  .039}{5} & \citet{salman2020adversarially}     & \textcolor[rgb]{ .039,  .039,  .039}{52.92 } & \cellcolor[rgb]{ .922,  .945,  .871} \textcolor[rgb]{ .039,  .039,  .039}{\textbf{24.6840 }} \\
		\bottomrule
	\end{tabular}
	\caption{The robust test accuracy (\%) achieved by the constructed AEs on the entire test set of ImageNet (containing 50,000 instances).}
	\label{tab:imagent}
\end{table}

\begin{table}[tbp]
	\centering
	\begin{tabular}{lc}
		\toprule
		\multicolumn{1}{c}{Attack} & Iteration Steps\\
		\midrule
		MT    &  128\\
		CE     &  160  \\
		CW    &  250 \\
		FAB   &  126\\
		\bottomrule
	\end{tabular}
	\caption{Component Attacks of the Constructed AE .}
	\label{tab:spcific_AE}
\end{table}%

We also use AutoAE to construct an AE for the top-1 CIFAR-10 ($l_{\infty},\epsilon=8/255$) defense model \citep{rebuffi2021fixing}  on the RobustBench Leaderboard.
The candidate attack pool and the annotated dataset are the same as before.
The constructed AE, shown in Table~\ref{tab:spcific_AE}, obtains a test robust accuracy of 46.3\% on this defense.

\bibliography{aaai23}

\end{document}